\documentclass[runningheads]{llncs}
\usepackage[frozencache,cachedir=minted-cache]{minted}
\usepackage[T1]{fontenc}
\usepackage{microtype}
\usepackage{subfigure}
\usepackage{booktabs}
\usepackage{graphicx}
\usepackage{hyperref}
\usepackage{algorithm}
\usepackage{algorithmic}

\hypersetup{
    colorlinks,
    linkcolor={black!50!black},
    citecolor={black!50!black},
    urlcolor={black!80!black}
}

\usepackage{amsmath}
\usepackage{amssymb}
\usepackage{mathtools}
\usepackage{listings}

\usepackage[capitalize,noabbrev]{cleveref}

\usepackage[utf8]{inputenc}
\usepackage{pgfplots}
\DeclareUnicodeCharacter{2212}{−}
\usepgfplotslibrary{groupplots,dateplot}
\usetikzlibrary{patterns,shapes.arrows}
\pgfplotsset{compat=newest}

\usepackage{todonotes}

\usepackage{adjustbox}

\DeclareMathOperator*{\argmin}{arg\,min}

\newcommand\blfootnote[1]{%
  \begingroup
  \renewcommand\thefootnote{}\footnote{#1}%
  \addtocounter{footnote}{-1}%
  \endgroup
}
\usepackage{color}

\begin{document}
\title{Sparse Private LASSO Logistic Regression}
\author{Amol Khanna\inst{1}\orcidID{0000-0002-5566-095X} \and 
Fred Lu\inst{1,2}\orcidID{0000-0003-1026-5734} \and 
Edward Raff\inst{1,2}\orcidID{0000-0002-9900-1972} \and
Brian Testa\inst{3}\orcidID{0000-0003-2349-9564}}
\authorrunning{A. Khanna et al.}
\institute{Booz Allen Hamilton, Annapolis Junction MD 20701, USA 
\email{\{Khanna\_Amol,Lu\_Fred,Raff\_Edward\}@bah.com} \and
University of Maryland, Baltimore County MD 21250, USA \and 
Air Force Research Laboratory, Rome NY 13441, USA \\
\email{Brian.Testa.1@us.af.mil}}
\maketitle              %
\begin{abstract}
LASSO regularized logistic regression is particularly useful for its built-in feature selection, allowing coefficients to be removed from deployment and producing sparse solutions. Differentially private versions of LASSO logistic regression have been developed, but generally produce dense solutions, reducing the intrinsic utility of the LASSO penalty. In this paper, we present a differentially private method for sparse logistic regression that maintains hard zeros. Our key insight is to first train a non-private LASSO logistic regression model to determine an appropriate privatized number of non-zero coefficients to use in final model selection. To demonstrate our method's performance, we run experiments on synthetic and real-world datasets. 

\end{abstract}

\section{Introduction}

\blfootnote{Approved for Public Release; Distribution Unlimited. PA Number: AFRL-2023-1816}Machine learning is increasingly being used in applications which require binary predictions while retaining the privacy of training data. For example, it can be used to speed up medical diagnosis and detect potential credit defaults, but each of these tasks employs sensitive datasets which must remain private when a model is deployed \cite{medical-machine-learning,financial-machine-learning}. 

One common statistical tool used for binary prediction is logistic regression (LR). It involves solving an empirical risk minimization problem to find a weight vector that can be used for prediction. Specifically, given a $p$-dimensional dataset $\{\mathbf{x}_1, \ldots, \mathbf{x}_n \} \in \mathbb{R}^{p}$ and labels $\{y_1, \ldots, y_n \} \in \{0, 1\}$, logistic regression solves 
\begin{equation}
\label{eq:1}
    \widehat{\mathbf{w}} \in \argmin_{\mathbf{w} \in \mathbb{R}^p} \frac{1}{n} \sum_{i = 1}^{n} -y_i \log  \sigma(\mathbf{w} \cdot \mathbf{x}_i) - (1 - y_i) \log (1 - \sigma(\mathbf{w} \cdot \mathbf{x}_i))
\end{equation}
where $\sigma(u) = \frac{1}{1 + \exp \{-u\}}$ is the sigmoid function. After finding $\widehat{\mathbf{w}}$, a future example $\mathbf{x}$ is classified as $1$ if $\sigma(\widehat{\mathbf{w}} \cdot \mathbf{x}) > 0.5$; otherwise, it is classified as $0$ \cite{logistic-regression}. The simplicity of the logistic regression's weight leads it to being interpretable, so it can be used in applications where explainable decisionmaking is required \cite{interpretable-ml,logistic-regression}.

In this work, we focus on the $L_1$-constrained (a.k.a., LASSO) setting. In its constrained form, LASSO reduces the $L_1$-norm of $\widehat{\mathbf{w}}$ by altering \hyperref[eq:1]{Eq. 1} to
\begin{equation}
    \label{eq:2}
    \widehat{\mathbf{w}} \in \argmin_{ \mathbf{w} \in \mathbb{R}^p:\ \lVert \mathbf{w} \rVert_1 \leq \lambda } \frac{1}{n} \sum_{i = 1}^{n} -y_i \log \sigma(\mathbf{w} \cdot \mathbf{x}_i) - (1 - y_i) \log (1 - \sigma(\mathbf{w} \cdot \mathbf{x}_i))
\end{equation}
where $\lambda$ is a hyperparameter controlling the constraint set \cite{constrained-lasso}. We can find $\widehat{\mathbf{w}}$ using the Frank-Wolfe algorithm, which is a projection-free method that iteratively takes steps in the direction of the largest negative gradient of a first-order approximation of the objective function \cite{frank-wolfe-original}. 

The solution to \hyperref[eq:1]{Eq. 2} is usually sparse, with exact $0$-values in the coefficients of $\widehat{\mathbf{w}}$. This is beneficial because sparse solutions are resistant to overfitting, enable higher-throughput deployments by discarding features with $0$-value coefficients, and allow for greater interpretability to practitioners \cite{lasso-advantages}.  

However, the model's weights may contain features which can be attacked to determine characteristics about the training data \cite{privacy-attacks}. In many applications of machine learning on sensitive datasets, such as those in medicine and finance, organizations must ensure the privacy of individuals in the training set, so using standard LR may not be possible. Instead, they can use differentially private LR. 

Differential privacy is a technique which provides a statistical guarantee of training data privacy. Specifically, given parameters $\epsilon$ and $\delta$, on any two datasets $D$ and $D'$ differing on one example, an approximate differentially private algorithm $\mathcal{A}$ satisfies $\Pr \left[ \mathcal{A}(D) \in O \right] \leq \exp\{ \epsilon\} \Pr \left[ \mathcal{A}(D') \in O \right] + \delta$ for any $O \subseteq \text{image}(\mathcal{A})$ \cite{dp-definition}. Note that lower values of $\epsilon$ and $\delta$ correspond to stronger privacy. 

Many different techniques have been developed for differentially private LR \cite{dp-logreg1,dp-logreg2,dp-logreg3,dp-logreg4}. However, implementations of LASSO-regularized LR with differential privacy tend to destroy the sparsity of $\widehat{\mathbf{w}}$, at which point there is little reason to use LASSO. We engineer a simple solution to rectify this in this work. %

 To the best of our knowledge, this is the first work to identify that outputs from differentially private LASSO LR algorithms can be dense and develop an intuitive method to produce a sparse LR weight from private LASSO. Results demonstrate our technique's sparsity on a variety of datasets, indicating its potential use in binary prediction applications which require sparse private weights. 

We are also the first to show that cross-entropy loss satisfies the conditions required by Talwar et al.'s private Frank-Wolfe LASSO linear regression algorithm so it can be used for logistic regression \cite{private-frank-wolfe}. We prove that the utility bound they find for linear regression extends to logistic regression. 

Finally, we extend our results to multinomial LR in \hyperref[sec:multinomial]{Appendix A}. To the best of our knowledge, this is the first work which has considered sparse private multinomial LR. After modifying Talwar et al.'s algorithm to handle a weight matrix instead of a weight vector, we prove its privacy. 

In \hyperref[sec:related-work]{Section 2}, we discuss works which have considered sparse private LR. In \hyperref[sec:methods]{Section 3} we outline our method and proofs for privacy and utility. \hyperref[sec:results]{Section 4} provides results for our method, demonstrating it is superior to prior works.  

\section{Related Work}
\label{sec:related-work}

In this work, we focus on developing a differentially private LASSO LR algorithm that produces sparse outputs. There are four works which have approached sparse differentially private regression. A summary of each is provided here. %

Kifer et al. use a feature selection algorithm to estimate a sparse support set for model parameters by finding the most important features on many partitions of the data. They then run a training algorithm with objective perturbation to produce a private weight on this support set. They name this the two-stage procedure \cite{two-stage}. This method can produce sparse solutions, but has some significant conditions. First, it requires that partitions of data agree on their sparse support set estimates; if they do not, which is often the case with high-dimensional datasets, the features it selects will be suboptimal. Additionally, their theory only holds for datsets which follow the restricted strong convexity and mutual incoherence properties, which is not true in many real-world datasets \cite{private-frank-wolfe}. %

Talwar et al. develop a private variant of the Frank-Wolfe algorithm using the exponential mechanism \cite{private-frank-wolfe}. They show their algorithm can be used for LASSO-regularized linear regression; we prove it can also be extended to LASSO-regularized logistic regression with the same utility guarantee. We use this algorithm directly in our solution; however, their na\"ive method produces a dense solution for two reasons: 

\begin{enumerate}
    \item 
    Prior to optimization, they initialize their weight randomly within the constraint set. Since private Frank-Wolfe is a greedy algorithm which only modifies the component with highest (noised) absolute gradient per iteration, components with negligible gradients that may not be useful for prediction will not be optimized. However, due to nonzero initialization, these coefficients will remain nonzero in the final solution. 
    \item 
    
    The above issue can be rectified by setting the initial weight to $\mathbf{0}$. This does not disturb the proof of privacy presented in the paper. But even after doing this, the gradient perturbation for private optimization causes many coefficients to be modified away from zero, significantly reducing sparsity. We address this issue by implementing a post-processing technique to set an appropriate number of the nonzero coefficients produced by the private Frank-Wolfe algorithm back to zero. 
\end{enumerate}

Wang \& Zhang create a private ADMM method using objective perturbation and apply their technique to LASSO- and $L_{1/2}$-regularized LR \cite{admm}. They showed that their algorithm was able to maintain sparsity on a synthetic dataset with 10,000 examples and 100 features; however, we were unable to replicate this result. It was difficult to use their algorithm because it has six hyperparameters which can affect the privacy, sparsity, and accuracy of the weight vector. The method we present in this work has three hyperparameters which are all set heuristically. 

 Another recent work has approached sparse private LR by optimizing an $L_0$-constrained learning problem \cite{dp-ight}. This method produces sparse solutions by applying differentially private stochastic descent followed by iterative gradient hard thresholding (IGHT) for many iterations \cite{dp-sgd}. While the authors claim they achieve superior utility to the private Frank-Wolfe algorithm, there are three conditions to their result. First, they assume that the design matrix $\mathbf{X}$ is drawn from a sub-gaussian distribution, which real-world datasets may not follow. Second, their guarantees rely on the loss function's Lipschitz constant with respect to the $L_2$ norm, which is often much worse than $L_1$-Lipschitzness \cite{private-frank-wolfe}. Finally, their big-$\mathcal{O}$ summary hides their algorithm's dependence on the condition number of the objective function's Hessian matrix. In high dimensional problems where the condition number of the Hessian is large, the noise added by gradient perturbation will often overwhelm the learned components, causing the algorithm to set these components back to 0 and lose learned information. Our results show that we outperform $L_0$-constrained private LR. 

\section{Methods}
\label{sec:methods}

We will now detail our new methods for private LASSO LR. First, we will review the standard private Frank-Wolfe algorithm for private LASSO LR and extend prior proofs to the logistic loss. Given this theoretically grounded foundation, we leverage the result as a sub-routine in a new \hyperref[alg:sparsifier]{\texttt{Sparsifier}} algorithm that is simple to implement and considerably more effective at producing sparse solutions. 

\subsection{Private LASSO for Logistic Regression}
Our technique for sparse private LR relies on the private LASSO method using the Frank-Wolfe algorithm. This method is shown in \hyperref[alg:private-frank-wolfe]{\texttt{Private LASSO}}. 

\begin{algorithm}[t]
    \caption{\texttt{Private LASSO}}
    \label{alg:private-frank-wolfe}
    \begin{algorithmic}
        \REQUIRE Privacy Parameters: $\epsilon > 0, 0 < \delta \leq 1$; Constraint Parameter: $\lambda > 0$; Iteration Parameter: $T$; Dataset: $D$ where $\lVert \mathbf{x}_i \rVert_\infty \leq 1$; Loss Function: $\mathcal{L}(\mathbf{w}; D) = \frac{1}{n} \sum_{i = 1}^{n} \mathcal{L}(\mathbf{w}; d_i)$ where $\mathcal{L}(\mathbf{w}; d_i)$ is $L$-Lipschitz with respect to the $L_1$ norm. 
        \STATE
        \STATE $S \gets \lambda * \{\text{Vertices of the Unit } L_1 \text{ Ball}\}$
        \STATE $\widehat{\mathbf{w}}_1 \gets \mathbf{0}$
        \FOR{$t = 1$ to $T - 1$} 
            \FOR{$\mathbf{s} \in S$}
                \STATE $\alpha_\mathbf{s} \gets \langle \mathbf{s}, \nabla \mathcal{L}(\widehat{\mathbf{w}}_t; D) \rangle + \text{Lap} \left( \frac{\lambda L \sqrt{8T \log(1/\delta)}}{n \epsilon} \right)$
            \ENDFOR
            \STATE $\widetilde{\mathbf{w}}_t \gets \argmin_{\mathbf{s} \in S} \alpha_\mathbf{s}$
            \STATE $\widehat{\mathbf{w}}_{t + 1} \gets (1 - \mu_t)\widehat{\mathbf{w}}_t + \mu_t \widetilde{\mathbf{w}}_t \text{ where } \mu_t = \frac{2}{t + 2}$
        \ENDFOR
        \STATE Output $\widehat{\mathbf{w}}_T$
    \end{algorithmic}
\end{algorithm}

To use \hyperref[alg:private-frank-wolfe]{\texttt{Private LASSO}}, the dataset must satisfy $\lVert \mathbf{x}_i \rVert_\infty \leq 1$, and the component loss function $\mathcal{L}(\mathbf{w}; d_i)$ must be Lipschitz with respect to the $L_1$ norm. The first condition can be met by scaling each feature such that the maximum absolute value of each feature will be $1$. The second condition is specific to the loss function. \cite{private-frank-wolfe} show the Lipschitzness of the mean-squared error loss; here, we show it for the binary cross-entropy loss. 

\begin{corollary}
\label{thm:lipschitzness}
When each data point $\mathbf{x}_i$ satisfies $\lVert \mathbf{x}_i \rVert_\infty \leq 1$, $\mathcal{L}(\mathbf{w}; d_i) = - y_i \log \sigma(\mathbf{w} \cdot \mathbf{x}_i) - (1 - y_i) \log (1 - \sigma (\mathbf{w} \cdot \mathbf{x}_i))$ has Lipschitz constant $1$ with respect to the $L_1$ norm. 
\end{corollary}
\begin{proof}
Shalev-Shwartz proves that if $\mathcal{L}(\mathbf{w}; d_i)$ is a convex function, its Lipschitz constant with respect to the $L_1$ norm is equivalent to the maximum value of the $L_\infty$ norm of its subgradient in our feasible set \cite{convex-lipschitz}. $\mathcal{L}(\mathbf{w}; d_i)$ is a convex function, and since it is differentiable, its subgradient at any point is simply equal to its gradient at the same point \cite{logistic-regression}. 

It can be shown that $\nabla \mathcal{L}(\mathbf{w}; d_i) = (\widehat{y} - y_i)\mathbf{x}_i$, where $\widehat{y} = \sigma(\mathbf{w} \cdot \mathbf{x}_i)$ \cite{logistic-regression}. Since $|\widehat{y} - y_i|$ can be a maximum of $1$ and we scale $\mathbf{x}_i$ so each of its components is between $-1$ and $1$, the maximum value of $\lVert (\widehat{y} - y_i)\mathbf{x}_i \rVert_\infty$ is $1$. This implies that $\mathcal{L}(\mathbf{w}; d_i)$ is $1$-Lipschitz with respect to the $L_1$ norm. 
\end{proof}

We also show that using \hyperref[alg:private-frank-wolfe]{\texttt{Private LASSO}} with the cross-entropy loss satisfies the same utility guarantee which Talwar et al. proved for the mean squared error loss \cite{private-frank-wolfe}. We present the curvature constant and two corollaries which are proved in prior works before showing our utility bound and its proof. 

\begin{definition}
    Let $\mathcal{C}$ be the feasible set of $\mathbf{w}$. For $\mathcal{L}: \mathcal{C} \rightarrow \mathbb{R}$, define the curvature constant $\Gamma_\mathcal{L}$ as below:
    $$\Gamma_\mathcal{L} = \sup_{\mathclap{\substack{\gamma \in (0, 1], \\ \mathbf{w}_1, \mathbf{w}_2 \in \mathcal{C}, \\ \mathbf{w}_3 = \mathbf{w}_1 + \gamma(\mathbf{w}_2 - \mathbf{w}_1)}}} \ \ \frac{2}{\gamma^2}(\mathcal{L}(\mathbf{w}_3) - \mathcal{L}(\mathbf{w}_1) - \langle \mathbf{w}_3 - \mathbf{w}_1, \nabla \mathcal{L}(\mathbf{w}_1) \rangle)$$
\end{definition}

\begin{corollary}
\label{cor:curvature}
    When $\mathcal{C} = \{\mathbf{w} \in \mathbb{R}^p:\ \lVert \mathbf{w} \rVert_1 \leq \lambda \}$, $\Gamma_\mathcal{L}$ is upper bounded by $\beta \lambda^2$, where $$\beta = \max_{\mathbf{w} \in \mathcal{C}, \lVert \mathbf{v} \rVert_1 = 1} \lVert \nabla^2 \mathcal{L}(\mathbf{w}) \cdot \mathbf{v} \rVert_\infty.$$ 
\end{corollary}
\begin{proof}
    This is a direct application of Remark 1 in Talwar et al.'s extended paper describing the private Frank-Wolfe algorithm, which was originally developed by Clarkson and Jaggi \cite{private-frank-wolfe,remark-1,remark-2}. 
\end{proof}

\begin{corollary}
\label{cor:utility}
    Let $L$ be the Lipschitz constant of $\mathcal{L}(\mathbf{w}; d_i)$, $p$ be the number of features in the dataset, $\mathcal{C}$ be $\{\mathbf{w} \in \mathbb{R}^p:\ \lVert \mathbf{w} \rVert_1 \leq \lambda \}$, and $\Gamma_\mathcal{L}$ be an upper bound on the curvature constant defined above. Then if we set $T = \frac{\Gamma_\mathcal{L}^{2/3} (n\epsilon)^{2/3}}{(L \lambda)^{2/3}}$,
    $$\mathbb{E}[\mathcal{L}(\widehat{\mathbf{w}}; D) - \min_{\mathbf{w}^* \in \mathcal{C}} \mathcal{L}(\mathbf{w}^*; D)] = \mathcal{O}\left(\frac{\Gamma_\mathcal{L}^{\frac{1}{3}}(L\lambda)^{\frac{2}{3}}\log(\frac{np}{\delta})}{(n\epsilon)^{\frac{2}{3}}}\right).$$
\end{corollary}
\begin{proof}
    This was proved by Talwar et al \cite{private-frank-wolfe}.
\end{proof}

\begin{theorem}
    When $\mathcal{C}$ is the $L_1$ unit ball, the output $\widehat{\mathbf{w}}$ of \hyperref[alg:private-frank-wolfe]{\texttt{Private LASSO}} ensures the following: 
    $$\mathbb{E}[\mathcal{L}(\widehat{\mathbf{w}}; D) - \min_{\mathbf{w} \in \mathcal{C}} \mathcal{L}(\mathbf{w}; D)] = \mathcal{O}\left( \frac{\log(\frac{np}{\delta})}{(n\epsilon)^{\frac{2}{3}}} \right)$$
    when $T$ is appropriately chosen. This is analogous to the utility bound proved for linear regression in \cite{private-frank-wolfe}. 
\end{theorem}
\begin{proof}
    This is an application of \hyperref[cor:utility]{Corollary 3}. \hyperref[thm:lipschitzness]{Corollary 1} shows that the Lipschitz constant of the component loss function is $\mathcal{O}(1)$. Now we show that the curvature constant for binary cross-entropy is also bounded in $\mathcal{O}(1)$ using \hyperref[cor:curvature]{Corollary 2}. 

    First, note that $\nabla^2 \mathcal{L}(\mathbf{w}) = \frac{1}{n} \mathbf{X}^\top \mathbf{S} \mathbf{X}$, where $\mathbf{X} \in \mathbb{R}^{n \times p}$ is the design matrix where each row is an element of the training data and $\mathbf{S}$ is a diagonal matrix where $s_{ii} = \sigma(\mathbf{x}_{i} \cdot \mathbf{w})(1 - \sigma(\mathbf{x}_{i} \cdot \mathbf{w}))$ \cite{logreg-hessian}. Let $\mathbf{z} = \mathbf{X}^\top\mathbf{SXv}$. Then, using the definition of matrix multiplication, it can be seen that $$z_{i} = \frac{1}{n} \sum_{l = 1}^{p} v_l \sum_{k = 1}^{n} x_{ki} \sigma (\mathbf{x}_k \cdot \mathbf{w})(1 - \sigma (\mathbf{x}_k \cdot \mathbf{w})) x_{kl},$$ 

    where $\mathbf{x}_k$ denotes the $k^\text{th}$ sample of the dataset, or equivalently the $k^\text{th}$ row of $\mathbf{X}$. Our goal is to find the maximum $L_\infty$ norm of $\mathbf{z}$ where $\lVert \mathbf{v} \rVert_1 = 1$. Specifically, we want to find $$\max_{\substack{\lVert \mathbf{v} \rVert_1 = 1 \\ i \in \{1, \ldots, d \}}} \left\lvert \frac{1}{n} \sum_{l = 1}^{p} v_l \sum_{k = 1}^{n} x_{ki} \sigma (\mathbf{x}_k \cdot \mathbf{w})(1 - \sigma (\mathbf{x}_k \cdot \mathbf{w})) x_{kl} \right\rvert.$$

    Note that for the objective function, irrespective of the value of $i$, 
    \begin{align*}
        &\left\lvert \frac{1}{n} \sum_{l = 1}^{p} v_l \sum_{k = 1}^{n} x_{ki} \sigma (\mathbf{x}_k \cdot \mathbf{w})(1 - \sigma (\mathbf{x}_k \cdot \mathbf{w})) x_{kl} \right\rvert \\ 
        &\leq \frac{1}{n} \sum_{l = 1}^{p} \left\lvert v_l \sum_{k = 1}^{n} x_{ki} \sigma (\mathbf{x}_k \cdot \mathbf{w})(1 - \sigma (\mathbf{x}_k \cdot \mathbf{w})) x_{kl} \right\rvert \\ 
        &\leq \frac{1}{n} \sum_{l = 1}^{p} \left\lvert v_l \right\rvert \left\lvert \sum_{k = 1}^{n} x_{ki} \sigma (\mathbf{x}_k \cdot \mathbf{w})(1 - \sigma (\mathbf{x}_k \cdot \mathbf{w})) x_{kl} \right\rvert \\ 
        &\leq \frac{1}{n} \sum_{l = 1}^{p} \left\lvert v_l \right\rvert \sum_{k = 1}^{n} \left\lvert x_{ki} \sigma (\mathbf{x}_k \cdot \mathbf{w})(1 - \sigma (\mathbf{x}_k \cdot \mathbf{w})) x_{kl} \right\rvert \\
        &\leq \frac{1}{n} \sum_{l = 1}^{p} \left\lvert v_l \right\rvert \sum_{k = 1}^{n} \left\lvert x_{ki} \right\rvert \left\lvert \sigma (\mathbf{x}_k \cdot \mathbf{w}) (1 - \sigma (\mathbf{x}_k \cdot \mathbf{w})) \right\rvert \left\lvert x_{kl} \right\rvert \\ 
        &\leq \frac{1}{n} \sum_{l = 1}^{p} \left\lvert v_l \right\rvert \sum_{k = 1}^{n} \frac{1}{4} 
        \leq \frac{1}{4} \sum_{l = 1}^{p} \left\lvert v_l \right\rvert 
        \leq \frac{1}{4}
    \end{align*}
    This means that each component of $\mathbf{z}$ must have magnitude less than $\frac{1}{4}$, meaning that $\beta \leq \frac{1}{4}$. Since we are constrained to the unit $L_1$ ball, $\lambda = 1$. Thus $\Gamma_\mathcal{L} \leq \beta \lambda^2$ is bounded in $\mathcal{O}(1)$. Plugging this into \hyperref[cor:utility]{Corollary 3} gives the desired result. 
\end{proof}

\subsection{Sparse Private LASSO for Logistic Regression}

As mentioned previously, the private Frank-Wolfe algorithm does not maintain sparse weights. For this reason, we developed an intuitive approach to sparsify its outputs. First, we modify it to initialize $\widehat{\mathbf{w}}_1$ to $\mathbf{0}$, which avoids nonzero components which are not updated during training. This is reflected in \hyperref[alg:private-frank-wolfe]{\texttt{Private LASSO}}. Second, we compute a value $c$, which is a clipped private count of the number of nonzero components of the true nonprivate solution. We then run the private LASSO algorithm and retain only the $c$ absolute largest nonzero components, setting the others back to zero. This helps us reduce the number of nonzero components which were updated away from zero due to noisy gradients. 

Our algorithm can be seen in \hyperref[alg:sparsifier]{\texttt{Sparsifier}}. For logistic regression, we use $\mathcal{L}(\mathbf{w}; D) = \frac{1}{n} \sum_{i = 1}^{n} -y_i \log \sigma(\mathbf{w} \cdot \mathbf{x}_i) - (1 - y_i) \log (1 - \sigma(\mathbf{w} \cdot \mathbf{x}_i))$, the binary cross-entropy loss. Note that in our implementation, the nonprivate LASSO algorithm used to find $\widehat{\mathbf{w}}_\mathit{NP}$ is identical to \hyperref[alg:private-frank-wolfe]{\texttt{Private LASSO}} except it is run for many iterations and does not add noise when computing $\alpha_\mathbf{s}$.  Additionally, we use clipping parameters $\alpha$ and $\beta$ to bound the sensitivity of the $c$ and ensure that it remains between these two values even after adding noise. We believe that this is conducive to most data analysis and engineering applications since the practitioner can set the minimum and maximum sparsity levels knowing the constraints of their system. 

Finally, we use a precision parameter $\rho$ to allow the user to find an appropriate balance between precision and recall for their application. Values of $\rho$ less than $1$ will make the algorithm keep fewer coefficients of $\widehat{\mathbf{w}}$ nonzero, which will improve the precision of $\widehat{\mathbf{w}}$'s nonzero coefficients but reduce its recall. Values of $\rho$ greater than $1$ will have the opposite effect.

\begin{algorithm}[t]
    \caption{\texttt{Sparsifier}} \label{alg:sparsifier}
    \begin{algorithmic}
        \REQUIRE Privacy Parameters: $\epsilon_1 > 0, \epsilon_2 > 0, 0 < \delta \leq 1$; Constraint Parameter: $\lambda > 0$; Clipping Parameters: $\alpha < \beta$; Precision Parameter: $\rho > 0$; Iteration Parameter: $T$; Dataset: $D = \{d_i\}_{i = 1}^{n}$ where $d_i = (\mathbf{x}_i \in \mathbb{R}^p, y_i \in \{0, 1\})$; Loss Function: $\mathcal{L}(\mathbf{w}; D)$ 
        \STATE
        \STATE $\widehat{\mathbf{w}}_\mathit{NP} \gets $Nonprivate LASSO$(\lambda, D, \mathcal{L}(\mathbf{w}; D))$
        \STATE $c \gets \text{Count Nonzero Components}\left(\widehat{\mathbf{w}}_\mathit{NP}\right)$
        \STATE $c \gets \text{Clip}\left(c, \alpha, \beta \right)$
        \STATE $c \gets c + \text{Double-Geometric}\left(1 - \exp\{-\frac{\epsilon_1}{\beta - \alpha} \}\right)$ 
        \STATE $c \gets \text{Clip}\left(c, \alpha, \beta \right)$
        \STATE $c \gets \text{Clip}(c \cdot \rho, 0, d)$
        \STATE $\widehat{\mathbf{w}}_P \gets \text{Private LASSO}\left(\epsilon_2, \delta, \lambda, T, D, \mathcal{L}(\mathbf{w}; D) \right)$
        \STATE $\widehat{\mathbf{w}} \gets \text{Keep Nonzero}\left(\widehat{\mathbf{w}}_P, c\right)$
        \STATE Output $\widehat{\mathbf{w}}$
    \end{algorithmic}
\end{algorithm}

\begin{theorem}
\hyperref[alg:sparsifier]{\texttt{Sparsifier}} adds a constant to the utility bound for \hyperref[alg:private-frank-wolfe]{\texttt{Private LASSO}}.
\end{theorem}

\begin{proof}
From the linearity of expectation, we know that $\mathbb{E}[\mathcal{L}(\widehat{\mathbf{w}}; D) - \min_{\mathbf{w}^* \in \mathcal{C}} \mathcal{L}(\mathbf{w}^*; D)]$ can be expanded to $\mathbb{E}[\mathcal{L}(\widehat{\mathbf{w}}; D) - \mathcal{L}(\widehat{\mathbf{w}}_P; D)] + \mathbb{E}[\mathcal{L}(\widehat{\mathbf{w}}_P; D)] - \min_{\mathbf{w}^* \in \mathcal{C}} \mathcal{L}(\mathbf{w}^*; D)$. 

Focusing on the first term, let $w^{(i)}$ denote the $i^{\text{th}}$ largest element of $\mathbf{w}$. Using the Lipschitz property, bounding the smallest $p-c$ components of $\widehat{\mathbf{w}}_P$ by the largest among them, and then accounting for the norm constraint, we have
\begin{align*}
\big|\mathcal{L}(\widehat{\mathbf{w}}) - \mathcal{L}(\widehat{\mathbf{w}}_P)\big| &\leq L \big\lVert \widehat{\mathbf{w}} - \widehat{\mathbf{w}}_P\big\rVert_1 = L \sum_{j=1}^p \big|\widehat{w}_j - \widehat{w}_{P_{j}}\big| \\
&\leq   L (p-c) \big|\widehat{w}_P^{(c+1)}\big| \leq L (p-c) \cdot \frac{\lambda}{c+1}
\end{align*}
but since this function decreases as $c$ decreases, we can plug in $\alpha$ for $c$. Additionally, we know that $\big\lVert \widehat{\mathbf{w}} - \widehat{\mathbf{w}}_P\big\rVert_1 \leq \lambda$. Then $\mathbb{E}[\mathcal{L}(\widehat{\mathbf{w}}; D) - \mathcal{L}(\widehat{\mathbf{w}}_P; D)] \leq L\lambda \cdot \min \Big\{ \frac{p - \alpha}{\alpha + 1}, \ 1 \Big\}$. This shows in the worst case, if all the selected coefficients are incorrect, applying \hyperref[alg:sparsifier]{\texttt{Sparsifier}} incurs a bias bounded by $L\lambda$ in expectation to \hyperref[alg:private-frank-wolfe]{\texttt{Private LASSO}}.\end{proof}

To derive a full utility bound, the second term in the expanded expectation is bounded by \hyperref[cor:utility]{Corollary 3}. The result is immediate. 

\begin{theorem}
\hyperref[alg:sparsifier]{\texttt{Sparsifier}} is differentially private with parameters $(\epsilon_1 + \epsilon_2, \delta)$. 
\end{theorem}

\begin{proof}
The output $\widehat{\mathbf{w}}$ of \hyperref[alg:sparsifier]{\texttt{Sparsifier}} is computed from two parameters: $c$ and $\widehat{\mathbf{w}}_P$. $c$ is a count variable with sensitivity bounded by $\beta - \alpha$ which has been privatized by adding noise from a double-geometric distribution with parameter $1 - \exp\{-\frac{\epsilon_1}{\beta - \alpha} \}$. This distribution produces lower variance than Laplacian noise but still ensures that $c$ is $(\epsilon_1, 0)$-differentially private \cite{double-geometric-2,double-geometric-1}. Although its value may change after the noise is added, these changes do not rely on information contained within the dataset. The post-processing property of differential privacy proves that transformations of differentially private outputs by functions which do not use the private dataset remain differentially private \cite{post-processing-dp}. 

Talwar et al. prove that $\widehat{\mathbf{w}}_P$ is $(\epsilon_2, \delta)$-differentially private \cite{private-frank-wolfe}.

According to the basic composition theorem of differential privacy, computing $\widehat{\mathbf{w}}$ from parameters with privacy guarantees $(\epsilon_1, 0)$ and $(\epsilon_2, \delta)$ produces a privacy guarantee of $(\epsilon_1 + \epsilon_2, \delta)$ for the output \cite{dp-composition}. 
\end{proof}

\hyperref[alg:private-frank-wolfe]{\texttt{Private LASSO}} and \hyperref[alg:sparsifier]{\texttt{Sparsifier}} can be efficiently implemented for high dimensional problems using computational linear algebra systems with sparse matrices to avoid using an inner loop for \hyperref[alg:private-frank-wolfe]{\texttt{Private LASSO}}. Psuedocode for implementing this using NumPy and SciPy is provided in \hyperref[sec:binary-implementation]{Appendix B} \cite{numpy,scipy}. 

\hyperref[alg:private-frank-wolfe]{\texttt{Private LASSO}} can be modified to achieve private LASSO-regularized multinomial logistic regression. This algorithm is shown in \hyperref[sec:multinomial]{Appendix A}. 

\section{Results}
\label{sec:results}

We begin by comparing the sparsity provided by \hyperref[alg:sparsifier]{\texttt{Sparsifier}} to solutions produced by \hyperref[alg:private-frank-wolfe]{\texttt{Private LASSO}}. After demonstrating that \hyperref[alg:sparsifier]{\texttt{Sparsifier}} performs better than \hyperref[alg:private-frank-wolfe]{\texttt{Private LASSO}}, we compare \hyperref[alg:sparsifier]{\texttt{Sparsifier}} to the methods for sparse private LR discussed in \hyperref[sec:related-work]{Section 2}. 

\subsection{Comparing \hyperref[alg:sparsifier]{\texttt{Sparsifier}} with \hyperref[alg:private-frank-wolfe]{\texttt{Private LASSO}}}

 To compare the performance of \hyperref[alg:sparsifier]{\texttt{Sparsifier}} with \hyperref[alg:private-frank-wolfe]{\texttt{Private LASSO}}, we trained both on a variety of binary prediction tasks. Note that in the following results, experiments on \hyperref[alg:private-frank-wolfe]{\texttt{Private LASSO}} set $\epsilon = 1, \delta = 1 / n_{\text{train}}$ and experiments on \hyperref[alg:sparsifier]{\texttt{Sparsifier}} set $\epsilon_1 = 0.05, \epsilon_2 = 0.95, \delta = 1 / n_{\text{train}}, \rho = 1$, producing the same privacy guarantee for both. For \hyperref[alg:sparsifier]{\texttt{Sparsifier}}, the nonprivate Frank-Wolfe algorithm was run for 50,000 iterations to achieve a nearly optimal result. Additionally, we heuristically set $\alpha = \sqrt{p}$ and $\beta = 2\sqrt{p}$ because many applications of sparse classifiers in data science seek $\mathcal{O}\left(\sqrt{p}\right)$ total nonzero features. To standardize results, 50 trials of each algorithm were run with $\lambda=10$ on all datasets. Note that the Synthetic and KDDCUP'99 datasets are described in \hyperref[sec:4.2]{Section 4.2} and all other datasets are taken directly from LIBSVM \cite{LIBSVM}. Datasets are arranged by number of features, with KDDCUP'99 having 41 and Gisette having 5000.

\subsubsection{Establishing the Density of \hyperref[alg:private-frank-wolfe]{\texttt{Private LASSO}}}

Here we demonstrate that \hyperref[alg:private-frank-wolfe]{\texttt{Private LASSO}} produces dense solutions to classification problems even when $\widehat{\mathbf{w}}_1$ is initialized to $\mathbf{0}$. To do this, we found the ratio of the number of nonzero coefficients found by \hyperref[alg:private-frank-wolfe]{\texttt{Private LASSO}} and \hyperref[alg:sparsifier]{\texttt{Sparsifier}} to the nonprivate Frank-Wolfe algorithm when learning to predict on a variety of datasets. 

\hyperref[tab:20-newsgroups-sparsity]{Table 1} lists the results of this experiment. It is clear that \hyperref[alg:private-frank-wolfe]{\texttt{Private LASSO}} produces far more nonzero coefficients than the nonprivate algorithm and that \hyperref[alg:sparsifier]{\texttt{Sparsifier}} mitigates this issue. This confirms the benefit of using \hyperref[alg:sparsifier]{\texttt{Sparsifier}} to consistently produce a truly sparse weight. 

\begin{table}[t]
\caption{Average number of nonzeros (NZ) and ratio of the number of nonzero coefficients found by the \hyperref[alg:private-frank-wolfe]{\texttt{Private LASSO}} and \hyperref[alg:sparsifier]{\texttt{Sparsifier}} to the nonprivate Frank-Wolfe algorithm on multiple datasets. \hyperref[alg:private-frank-wolfe]{\texttt{Private LASSO}} and \hyperref[alg:sparsifier]{\texttt{Sparsifier}} were run for 1000 iterations, while the nonprivate Frank-Wolfe algorithm was run for 50,000 iterations. }
\vskip -0.10 in
\label{tab:20-newsgroups-sparsity}
\begin{center}
\begin{small}
\begin{sc}
\setlength\tabcolsep{6.5pt}
\begin{tabular}{lccccc}
\toprule
& Nonprivate & \multicolumn{2}{c}{\hyperref[alg:private-frank-wolfe]{\texttt{Private LASSO}}} & \multicolumn{2}{c}{\hyperref[alg:sparsifier]{\texttt{Sparsifier}}} \\ 
\cmidrule(lr){2-2} \cmidrule(lr){3-4} \cmidrule(lr){5-6} 
Dataset & NZ & NZ & Ratio & NZ & Ratio \\
\midrule
KDDCUP'99 & 4 & 106.90 & 26.75 & 16.10 & 4.29 \\ 
Covtype & 17 & 30.38 & 3.18 & 11.78 & 0.68 \\
Splice & 52 & 60.00 & 1.15 & 12.70 & 0.24 \\
Phishing & 2 & 68.00 & 34.00 & 12.02 & 6.16 \\
Synthetic & 7 & 100.00 & 14.29 & 15.08 & 2.29 \\
Mushroom & 16 & 111.98 & 7.00 & 16.02 & 1.04 \\
A9A & 19 & 122.90 & 6.47 & 18.06 & 0.89 \\
W8A & 20 & 285.64 & 14.44 & 25.60 & 1.35 \\ 
Colon & 58 & 785.72 & 13.56 & 68.16 & 1.12 \\
Gisette & 77 & 901.30 & 11.67 & 102.50 & 1.43 \\
\bottomrule
\end{tabular}
\end{sc}
\end{small}
\end{center}
\vskip -0.20in
\end{table}

\subsubsection{Comparing \hyperref[alg:sparsifier]{\texttt{Sparsifier}} with \hyperref[alg:private-frank-wolfe]{\texttt{Private LASSO}} Run for Fewer Iterations}

The previous result indicates that when run for the same number of iterations, \hyperref[alg:sparsifier]{\texttt{Sparsifier}} produces a more sparse result than \hyperref[alg:private-frank-wolfe]{\texttt{Private LASSO}}. However, \hyperref[alg:private-frank-wolfe]{\texttt{Private LASSO}} can produce sparse solutions if run for fewer iterations, as this will cause the algorithm to update fewer coefficients. Additionally, running it for fewer iterations will add less noise and thus produce more accurate updates. 

We wanted to determine if the information that \hyperref[alg:sparsifier]{\texttt{Sparsifier}} gains from running \hyperref[alg:private-frank-wolfe]{\texttt{Private LASSO}} for more iterations and sparsifying the result is worth the more inaccurate updates compared to only running \hyperref[alg:private-frank-wolfe]{\texttt{Private LASSO}} for fewer iterations to achieve the same sparsity. To answer this question, we found the number of iterations to run \hyperref[alg:private-frank-wolfe]{\texttt{Private LASSO}} to achieve approximately the same sparsity as running \hyperref[alg:sparsifier]{\texttt{Sparsifier}} for 1000 iterations on each dataset. We compared the accuracy and AUC of the solutions. 

Results are displayed in \hyperref[tab:sparse-accuracy-comparison]{Table 2}. \hyperref[alg:sparsifier]{\texttt{Sparsifier}} produces significantly better accuracy and AUC for most datasets, while \hyperref[alg:private-frank-wolfe]{\texttt{Private LASSO}} is never signficantly better than \hyperref[alg:sparsifier]{\texttt{Sparsifier}}. This implies that it is worth running \hyperref[alg:private-frank-wolfe]{\texttt{Private LASSO}} for more iterations and sparsifying the result instead of just running it for fewer iterations, which validates the usefulness of \hyperref[alg:sparsifier]{\texttt{Sparsifier}}. 

\begin{table}[t]
\caption{Average accuracies and AUCs of \hyperref[alg:sparsifier]{\texttt{Sparsifier}} and \hyperref[alg:private-frank-wolfe]{\texttt{Private LASSO}} when \hyperref[alg:private-frank-wolfe]{\texttt{Private LASSO}} was run for enough iterations to approximately match the number of nonzero coefficients (NZ) produced by \hyperref[alg:sparsifier]{\texttt{Sparsifier}} on each dataset. Values in bold are significantly better as measured by the Wilcoxon Signed-Rank test \cite{wilcoxon-1,wilcoxon-2}.}
\vskip -0.10 in
\label{tab:sparse-accuracy-comparison}
\begin{center}
\begin{small}
\begin{sc}
\setlength\tabcolsep{6.5pt}
\adjustbox{max width=\columnwidth}{
\begin{tabular}{@{}lcccccc@{}}
\toprule
\multicolumn{1}{l}{} & \multicolumn{3}{c}{\hyperref[alg:sparsifier]{\texttt{Sparsifier}}}           & \multicolumn{3}{c}{\hyperref[alg:private-frank-wolfe]{\texttt{Private LASSO}}}       \\ \cmidrule(lr){2-4} \cmidrule(lr){5-7}
Dataset &
  \multicolumn{1}{c}{NZ} &
  \multicolumn{1}{c}{Acc} &
  \multicolumn{1}{c}{AUC} &
  \multicolumn{1}{c}{NZ} &
  \multicolumn{1}{c}{Acc} &
  \multicolumn{1}{c}{AUC} \\ \midrule
KDDCUP'99                       & 17                   & 53.71          & 91.48          & 14                   & 54.66                   & 86.57                   \\
Covtype                     & 12                   & \textbf{57.35}                   & \textbf{64.36}          & 11                   & 54.90          & 60.15                   \\
Splice                      & 12                   & \textbf{71.34}          & \textbf{86.84}          & 14                   & 49.95                   & 50.13                   \\
Phishing                    & 12                   & 63.79          & \textbf{77.23}          & 10                   & 64.26                   & 68.95                   \\
Synthetic                   & 16                   & \textbf{85.17}          & \textbf{93.28}                   & 17                   & 56.41                   & 59.03          \\ 
Mushroom                    & 17                   & \textbf{77.89}          & \textbf{88.65}          & 17                   & 72.66                   & 83.23                    \\
A9A                         & 17                   & \textbf{76.60}          & 67.64          & 19                   & 71.57                   & 66.29                   \\
W8A                         & 27                   & \textbf{97.15}          & 75.64          & 26                   & 97.02                   & 76.36                   \\
Colon                       & 65                  & 83.67          & 85.00          & 69                  & 83.33                   & 86.61                   \\
Gisette                     & 110                  & \textbf{67.94}          & \textbf{79.76}          & 102                  & 64.48                   & 70.05                   \\ \bottomrule
\end{tabular}
}
\end{sc}
\end{small}
\end{center}
\vskip -0.20in
\end{table}

\subsection{Comparing \hyperref[alg:sparsifier]{\texttt{Sparsifier}} with Related Works}
\label{sec:4.2}
In this section, we demonstrate that \hyperref[alg:sparsifier]{\texttt{Sparsifier}} produces more sparse and accurate results than other approachs to private LASSO LR. 

In all the following experiments, we used hyperparameters $\epsilon_1 = 0.05, \rho = 1, \delta = 1 / n_{\text{train}}$. We again set $\alpha = \sqrt{p}$ and $\beta = 2\sqrt{p}$ and ran the nonprivate Frank-Wolfe algorithm required in \hyperref[alg:sparsifier]{\texttt{Sparsifier}} for 50,000 iterations. We chose $\lambda$ from 5 logarithmically-spaced values between 1 and 50 to maximize accuracy on a held-out validation set consisting of 20\% of the data. All reported $\epsilon$ values indicate the privacy of the output weights, $\epsilon_1 + \epsilon_2$. Since we were unable to replicate the sparse outputs of the private ADMM algorithm, we directly use the results provided in that paper. We explore the synthetic and KDDCUP'99 datasets since these are what were tested in the private ADMM paper \cite{admm}. 

\subsubsection{Synthetic Data}

\begin{figure}[t]
\begin{center}
\centerline{\begin{tikzpicture}

\definecolor{black25}{RGB}{25,25,25}
\definecolor{darkslategray38}{RGB}{38,38,38}
\definecolor{indianred1967882}{RGB}{196,78,82}
\definecolor{lavender234234242}{RGB}{234,234,242}
\definecolor{lightgray204}{RGB}{204,204,204}
\definecolor{steelblue76114176}{RGB}{76,114,176}

\begin{axis}[
axis background/.style={fill=lavender234234242},
axis line style={white},
legend cell align={left},
legend style={
  fill opacity=0.5,
  draw opacity=1,
  text opacity=1,
  at={(0.05,0.65)},
  anchor=west,
  draw=lightgray204,
  fill=lavender234234242, 
  nodes={scale=0.75, transform shape}
},
tick align=inside,
x grid style={white},
xlabel=\textcolor{darkslategray38}{\(\displaystyle \epsilon\)},
xmajorgrids,
xmajorticks=true,
xmin=0.85, xmax=4.15,
xtick style={draw=none},
y grid style={white},
ylabel=\textcolor{darkslategray38}{Average Number of Correct Zeros},
ymajorgrids,
ymajorticks=true,
ymin=0.05, ymax=99.95,
ytick style={draw=none}
]
\addplot [thick, indianred1967882]
table {%
1 85.0199966430664
1.5 84.0599975585938
2 85.3600006103516
2.5 85.0199966430664
3 85.5999984741211
3.5 84.3399963378906
4 85.7200012207031
};
\addlegendentry{\texttt{Sparsifier}}
\addplot [thick, steelblue76114176, dotted]
table {%
1 22.0799999237061
1.5 29.3999996185303
2 21.8199996948242
2.5 3.61999988555908
3 3.55999994277954
3.5 15.1599998474121
4 7.80000019073486
};
\addlegendentry{Two-Stage}
\addplot [thick, steelblue76114176, dashed]
table {%
1 15
1.5 24
2 32
2.5 39
3 51
4 60
};
\addlegendentry{Private ADMM}
\addplot [thick, steelblue76114176, dash pattern=on 1pt off 3pt on 3pt off 3pt]
table {%
1 0.980000019073486
1.5 4.65999984741211
2 19.0599994659424
2.5 1.79999995231628
3 24.9400005340576
3.5 13
4 66.3399963378906
};
\addlegendentry{Private IGHT}
\addplot [thick, black25]
table {%
1 92
4 92
};
\addlegendentry{Maximum Correct Zeros}
\end{axis}

\end{tikzpicture}}
\vskip -0.10 in
\caption{Comparing the number of correctly identified zeros of \hyperref[alg:sparsifier]{\texttt{Sparsifier}} with two-stage, private ADMM, and private IGHT. \hyperref[alg:sparsifier]{\texttt{Sparsifier}} and private IGHT were run for 1000 iterations. A black line at 92 coefficients has been drawn indicating the number of true zero coefficients. Higher numbers of correct zero coefficients indicate that an algorithm is able to maintain sparsity.}
\label{fig:correct-zeros}
\end{center}
\vskip -0.25in
\end{figure}
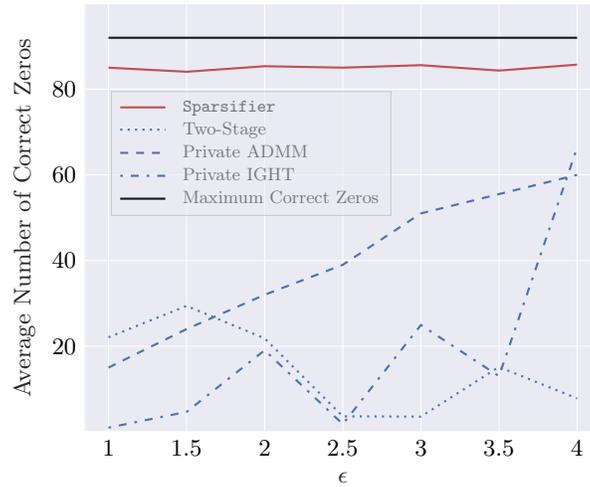

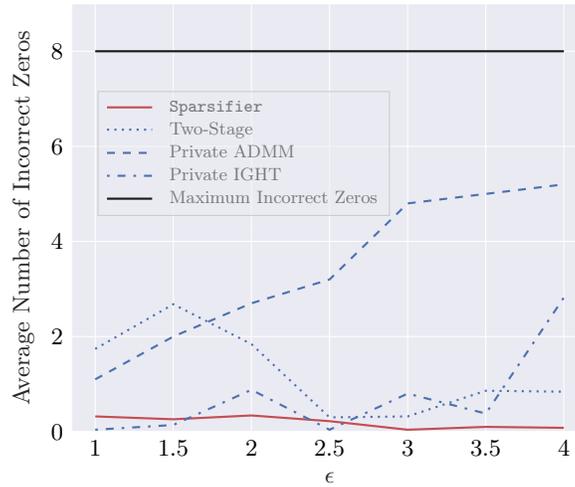
\begin{figure}[t]
\begin{center}
\centerline{\begin{tikzpicture}

\definecolor{black25}{RGB}{25,25,25}
\definecolor{darkslategray38}{RGB}{38,38,38}
\definecolor{indianred1967882}{RGB}{196,78,82}
\definecolor{lavender234234242}{RGB}{234,234,242}
\definecolor{lightgray204}{RGB}{204,204,204}
\definecolor{steelblue76114176}{RGB}{76,114,176}

\begin{axis}[
axis background/.style={fill=lavender234234242},
axis line style={white},
legend cell align={left},
legend style={
  fill opacity=0.5,
  draw opacity=1,
  text opacity=1,
  at={(0.05,0.65)},
  anchor=west,
  draw=lightgray204,
  fill=lavender234234242, 
  nodes={scale=0.75, transform shape}
},
tick align=inside,
x grid style={white},
xlabel=\textcolor{darkslategray38}{\(\displaystyle \epsilon\)},
xmajorgrids,
xmajorticks=true,
xmin=0.85, xmax=4.15,
xtick style={draw=none},
y grid style={white},
ylabel=\textcolor{darkslategray38}{Average Number of Incorrect Zeros},
ymajorgrids,
ymajorticks=true,
ymin=0, ymax=9,
ytick style={draw=none}
]
\addplot [thick, indianred1967882]
table {%
1 0.319999933242798
1.5 0.259999990463257
2 0.340000033378601
2.5 0.220000028610229
3 0.0399999618530273
3.5 0.100000023841858
4 0.0800000429153442
};
\addlegendentry{\texttt{Sparsifier}}
\addplot [thick, steelblue76114176, dotted]
table {%
1 1.74000000953674
1.5 2.6800000667572
2 1.8400000333786
2.5 0.299999952316284
3 0.319999933242798
3.5 0.860000014305115
4 0.839999914169312
};
\addlegendentry{Two-Stage}
\addplot [thick, steelblue76114176, dashed]
table {%
1 1.10000002384186
1.5 2
2 2.70000004768372
2.5 3.20000004768372
3 4.80000019073486
4 5.19999980926514
};
\addlegendentry{Private ADMM}
\addplot [thick, steelblue76114176, dash pattern=on 1pt off 3pt on 3pt off 3pt]
table {%
1 0.0399999618530273
1.5 0.139999985694885
2 0.879999995231628
2.5 0.0399999618530273
3 0.799999952316284
3.5 0.379999995231628
4 2.8199999332428
};
\addlegendentry{Private IGHT}
\addplot [thick, black25]
table {%
1 8
4 8
};
\addlegendentry{Maximum Incorrect Zeros}
\end{axis}

\end{tikzpicture}}
\vskip -0.10 in
\caption{Comparing the number of incorrectly identified zeros of \hyperref[alg:sparsifier]{\texttt{Sparsifier}} with two-stage, private ADMM, and private IGHT. \hyperref[alg:sparsifier]{\texttt{Sparsifier}} and private IGHT were run for 1000 iterations. A black line at 8 coefficients has been drawn indicating the number of true nonzero coefficients. Lower numbers of incorrect zero coefficients indicate that an algorithm is able to correctly identify features which are useful for the prediction task.}
\label{fig:incorrect-zeros}
\end{center}
\vskip -0.25in
\end{figure}

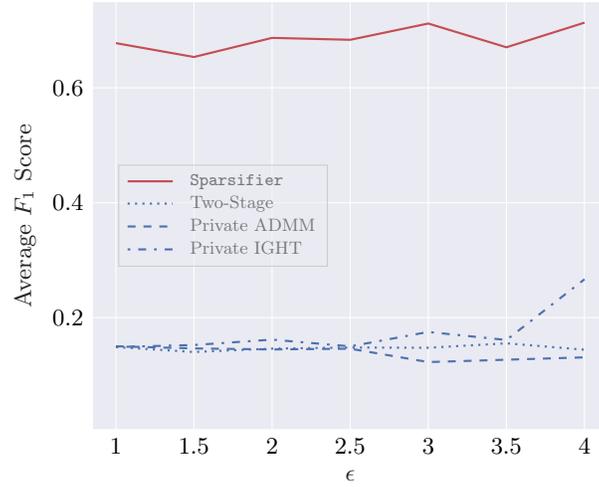
\begin{figure}[t]
\begin{center}
\centerline{\begin{tikzpicture}

\definecolor{darkslategray38}{RGB}{38,38,38}
\definecolor{indianred1967882}{RGB}{196,78,82}
\definecolor{lavender234234242}{RGB}{234,234,242}
\definecolor{lightgray204}{RGB}{204,204,204}
\definecolor{steelblue76114176}{RGB}{76,114,176}

\begin{axis}[
axis background/.style={fill=lavender234234242},
axis line style={white},
legend cell align={left},
legend style={
  fill opacity=0.5,
  draw opacity=1,
  text opacity=1,
  at={(0.05,0.5)},
  anchor=west,
  draw=lightgray204,
  fill=lavender234234242, 
  nodes={scale=0.75, transform shape}
},
tick align=inside,
x grid style={white},
xlabel=\textcolor{darkslategray38}{\(\displaystyle \epsilon\)},
xmajorgrids,
xmajorticks=true,
xmin=0.85, xmax=4.15,
xtick style={draw=none},
y grid style={white},
ylabel=\textcolor{darkslategray38}{Average \(\displaystyle F_1\) Score},
ymajorgrids,
ymajorticks=true,
ymin=0.005, ymax=0.75,
ytick style={draw=none}
]
\addplot [thick, indianred1967882]
table {%
1 0.67784643
1.5 0.65371622
2 0.68699552
2.5 0.68365554
3 0.71198569
3.5 0.67062818
4 0.71351351
};
\addlegendentry{\texttt{Sparsifier}}
\addplot [thick, steelblue76114176, dotted]
table {%
1 0.148728847503662
1.5 0.140147566795349
2 0.146075367927551
2.5 0.147963047027588
3 0.14752209186554
3.5 0.155251145362854
4 0.144122362136841
};
\addlegendentry{Two-Stage}
\addplot [thick, steelblue76114176, dashed]
table {%
1 0.150163173675537
1.5 0.146341443061829
2 0.144611120223999
2.5 0.145896673202515
3 0.122605323791504
4 0.130841135978699
};
\addlegendentry{Private ADMM}
\addplot [thick, steelblue76114176, dash pattern=on 1pt off 3pt on 3pt off 3pt]
table {%
1 0.148812890052795
1.5 0.152325630187988
2 0.161707878112793
2.5 0.149962306022644
3 0.175054669380188
3.5 0.161065340042114
4 0.266735315322876
};
\addlegendentry{Private IGHT}
\end{axis}

\end{tikzpicture}}
\caption{$F_1$ scores of the learned coefficients with respect to the true weight. In the calculation, true positives were correctly identified nonzero coefficients, false positives were learned nonzero coefficients in the latter 92 components of the weight, and false negatives were incorrectly identified zero coefficients. Higher scores indicate that the algorithm does at distinguishing between true nonzeros and zeros.}
\label{fig:f1}
\end{center}
\vskip -0.25in
\end{figure}

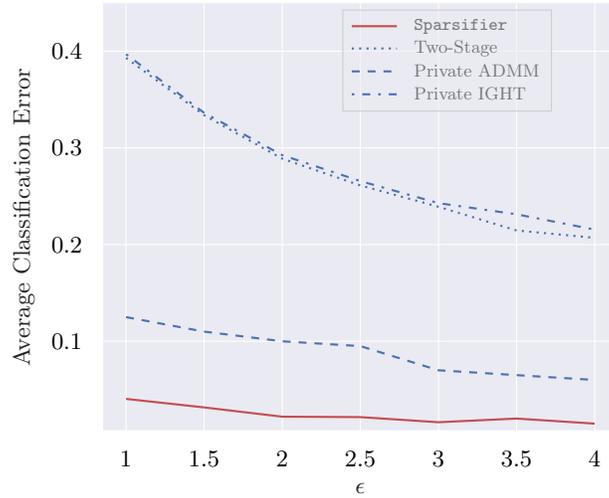
\begin{figure}[!h]
\begin{center}
\centerline{\begin{tikzpicture}

\definecolor{darkslategray38}{RGB}{38,38,38}
\definecolor{indianred1967882}{RGB}{196,78,82}
\definecolor{lavender234234242}{RGB}{234,234,242}
\definecolor{lightgray204}{RGB}{204,204,204}
\definecolor{steelblue76114176}{RGB}{76,114,176}

\begin{axis}[
axis background/.style={fill=lavender234234242},
axis line style={white},
legend cell align={left},
legend style={
  fill opacity=0.5,
  draw opacity=1,
  text opacity=1,
  at={(0.465,0.865)},
  anchor=west,
  draw=lightgray204,
  fill=lavender234234242, 
  nodes={scale=0.75, transform shape}
},
tick align=outside,
x grid style={white},
xlabel=\textcolor{darkslategray38}{\(\displaystyle \epsilon\)},
xmajorgrids,
xmajorticks=true,
xmin=0.85, xmax=4.15,
xtick style={draw=none},
y grid style={white},
ylabel=\textcolor{darkslategray38}{Average Classification Error},
ymajorgrids,
ymajorticks=true,
ymin=0.0075106999999999, ymax=0.45,
ytick style={draw=none}
]
\addplot [thick, indianred1967882]
table {%
1 0.0404200553894043
1.5 0.0316139459609985
2 0.0220279693603516
2.5 0.0215959548950195
3 0.0162899494171143
3.5 0.0200940370559692
4 0.0148559808731079
};
\addlegendentry{\texttt{Sparsifier}}
\addplot [thick, steelblue76114176, dotted]
table {%
1 0.393020033836365
1.5 0.333770036697388
2 0.288800001144409
2.5 0.261165976524353
3 0.239058017730713
3.5 0.214573979377747
4 0.206977963447571
};
\addlegendentry{Two-Stage}
\addplot [thick, steelblue76114176, dashed]
table {%
1 0.125
1.5 0.110000014305115
2 0.100000023841858
2.5 0.0950000286102295
3 0.0700000524520874
4 0.059999942779541
};
\addlegendentry{Private ADMM}
\addplot [thick, steelblue76114176, dash pattern=on 1pt off 3pt on 3pt off 3pt]
table {%
1 0.396836042404175
1.5 0.336166024208069
2 0.29246997833252
2.5 0.265624046325684
3 0.242831945419312
3.5 0.231292009353638
4 0.21547794342041
};
\addlegendentry{Private IGHT}
\end{axis}

\end{tikzpicture}}
\vskip -0.10 in
\caption{Comparing the accuracy of \hyperref[alg:sparsifier]{\texttt{Sparsifier}} with two-stage, private ADMM, and private IGHT. \hyperref[alg:sparsifier]{\texttt{Sparsifier}} and private IGHT were run for 1000 iterations.}
\label{fig:accuracy-comparison}
\end{center}
\vskip -0.25in
\end{figure}

 Here, we test \hyperref[alg:sparsifier]{\texttt{Sparsifier}} on a synthetic dataset to demonstrate that it outperforms prior solutions for sparse private logistic regression. We compare our solution to the the two-stage procedure developed by Kifer et al., private ADMM algorithm developed by Wang \& Zhang, and the private IGHT method developed by Wang \& Gu \cite{two-stage,admm,dp-ight}. 

For this comparison, we generated synthetic data according to the details provided by Wang \& Zhang \cite{admm}. Specifically, we generated $10,000$ datapoints where each $\mathbf{x}_i \in \mathbb{R}^{100}$ was drawn from $\mathcal{N}(\mathbf{0}, \Sigma)$ with $\Sigma_{ij} = 0.5^{|i - j|}$ and scaled each feature so its maximum absolute value in the dataset was $1$. The true weight vector $\mathbf{w}^*$ had its first 8 components set to $\begin{bmatrix} 10 & 9 & 8 & 7 & 6 & 5 & 4 & 0.5 \end{bmatrix}$ and its last 92 components set to $0$. Then, if $\sigma(\mathbf{w}^* \cdot \mathbf{x}_i) > 0.5$, $y_i = 1$, else $y_i = 0$. After reserving 20\% of this data for validation, we ran \hyperref[alg:sparsifier]{\texttt{Sparsifier}} 50 times for each value of $\epsilon$, choosing $\lambda$ which maximized average accuracy. We also ran the two-stage and private IGHT methods 50 times for each $\epsilon$. 

\hyperref[fig:correct-zeros]{Figure 1} and \hyperref[fig:incorrect-zeros]{Figure 2} display the results of this experiment. It is clear that \hyperref[alg:sparsifier]{\texttt{Sparsifier}} outperforms the other techniques in finding the correct zero coefficients while setting almost none of the true nonzero coefficients to zero. This means that on this dataset it can produce solutions while avoiding setting the true nonzero coefficients to zero. \hyperref[fig:f1]{Figure 3} summarizes these results, computing the $F_1$ score of the learned coefficients with respect to the true weight. 

We also compare the classification error of each method on a held-out test set in \hyperref[fig:accuracy-comparison]{Figure 4}. For all values of $\epsilon$, \hyperref[alg:sparsifier]{\texttt{Sparsifier}} achieves lower classification error than the other techniques. 

\subsubsection{KDDCUP'99}
\label{sec:4.3}

\begin{figure}[!h]
\begin{center}
\centerline{\begin{tikzpicture}

\definecolor{darkslategray38}{RGB}{38,38,38}
\definecolor{indianred1967882}{RGB}{196,78,82}
\definecolor{lavender234234242}{RGB}{234,234,242}
\definecolor{lightgray204}{RGB}{204,204,204}
\definecolor{steelblue76114176}{RGB}{76,114,176}

\begin{axis}[
axis background/.style={fill=lavender234234242},
axis line style={white},
legend cell align={left},
legend style={
  fill opacity=0.5,
  draw opacity=1,
  text opacity=1,
  at={(0.05,0.65)},
  anchor=west,
  draw=lightgray204,
  fill=lavender234234242, 
  nodes={scale=0.75, transform shape}
},
tick align=outside,
x grid style={white},
xlabel=\textcolor{darkslategray38}{\(\displaystyle \epsilon\)},
xmajorgrids,
xmajorticks=true,
xmin=0.9, xmax=3.1,
xtick style={draw=none},
y grid style={white},
ylabel=\textcolor{darkslategray38}{Average Classification Error},
ymajorgrids,
ymajorticks=true,
ymin=-0.0304629174209574, ymax=0.85,
ytick style={draw=none}
]
\addplot [thick, indianred1967882]
table {%
1 0.0110251903533936
1.5 0.00858020782470703
2 0.00854682922363281
2.5 0.00766849517822266
3 0.00761198997497559
};
\addlegendentry{\texttt{Sparsifier}}
\addplot [thick, steelblue76114176, dotted]
table {%
1 0.0221313238143921
1.5 0.0116636753082275
2 0.0112500190734863
2.5 0.010560154914856
3 0.00928926467895508
};
\addlegendentry{Two-Stage}
\addplot [thick, steelblue76114176, dashed]
table {%
1 0.120000004768372
1.5 0.0750000476837158
2 0.059999942779541
3 0.0499999523162842
};
\addlegendentry{Private ADMM}
\addplot [thick, steelblue76114176, dash pattern=on 1pt off 3pt on 3pt off 3pt]
table {%
1 0.730558395385742
1.5 0.722129821777344
2 0.745075106620789
2.5 0.759255409240723
3 0.740331411361694
};
\addlegendentry{Private IGHT}
\end{axis}

\end{tikzpicture}}
\caption{Comparing the accuracy of \hyperref[alg:sparsifier]{\texttt{Sparsifier}} with two-stage, private ADMM, and private IGHT. \hyperref[alg:sparsifier]{\texttt{Sparsifier}} and private IGHT were run for 1000 iterations.}
\label{fig:kdd-accuracy-comparison}
\end{center}
\vskip -0.35in
\end{figure}
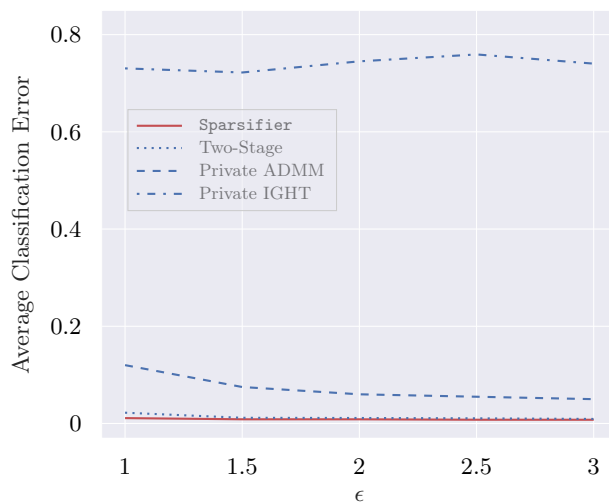

We also compared these methods on the KDDCUP'99 dataset, which consists of 5 million samples with 41 features and a binary positive or negative label. Positive labels correspond to benign network connections while negative labels correspond to malicious denial-of-service attacks. For our experiment, we followed the procedure of Wang \& Zhang \cite{admm}. Specifically, we randomly chose 60000 points for training, 38420 of which had positive labels. We then normalized each feature so its maximum absolute value was 1, and then normalized each training sample so its norm was one. After retaining 20\% of this data for validation, we followed the procedure described for the synthetic data. 

\hyperref[fig:kdd-accuracy-comparison]{Figure 5} demonstrates the results of this experiment on the test set. \hyperref[alg:sparsifier]{\texttt{Sparsifier}} and the two-stage procedure performed well on this data, while private ADMM and private IGHT had higher average classification error for all levels of privacy. 

 Comparing the performance of the two-stage and private IGHT algorithms in \hyperref[fig:accuracy-comparison]{Figure 4} and \hyperref[fig:kdd-accuracy-comparison]{Figure 5} also implies some of the fundamental drawbacks of these methods. Specifically, the two-stage procedure performs badly in \hyperref[fig:accuracy-comparison]{Figure 4} but well in \hyperref[fig:kdd-accuracy-comparison]{Figure 5}. We believe this is because the features in the synthetic dataset are correlated and thus violate the mutual incoherence condition whereas those in the KDDCUP'99 dataset are less correlated. The private IGHT algorithm does not perform well on either, which we believe indicates that the condition number of the Hessian matrix is high on both datasets. Since the condition number is high, gradient updates in some directions may be larger than some of the important coefficients within the weight vector, but due to hard thresholding, these coefficients will be set to $0$ after an inaccurate update. 

\section{Conclusion}

In this paper, we presented a method for sparse private logistic regression which rested on two techniques. First, we employed a private LASSO-regularized Frank-Wolfe algorithm to produce an initial prediction of the output weight. To the best of our knowledge, we are also the first to prove a utility bound of this algorithm when using the binary cross-entropy loss. Then we ensured that the weight is sparse by keeping only an appropriate number of its coefficients nonzero. Experiments on synthetic and real-world datasets demonstrate that our method is effective at finding accurate solutions despite privacy constraints. 

Note that while we focus on sparse private logistic regression, \hyperref[alg:sparsifier]{\texttt{Sparsifier}} is capable of produce a sparse solution for any loss function to \hyperref[alg:private-frank-wolfe]{\texttt{Private LASSO}}. For example, it would work for our private LASSO-regularized multinomial logistic regression algorithm shown in \hyperref[sec:multinomial]{Appendix A}. 

However, there is one main limitation to our method. When trying to have a very low cumulative $\epsilon$ privacy guarantee, the value of $\epsilon_1$ must be very small, which can cause wide changes in the sparsity of the final output weight and may set many of its useful coefficients to $0$. We addressed this as best we could by noising the count variable $c$ with the Double-Geometric distribution instead of the higher-variance Laplace distribution, but since privacy is inherently generated from added noise, this phenomenon is unavoidable. 
\bibliographystyle{splncs04}
\bibliography{bibliography}

\begin{thebibliography}{10}
\providecommand{\url}[1]{\texttt{#1}}
\providecommand{\urlprefix}{URL }
\providecommand{\doi}[1]{https://doi.org/#1}

\bibitem{dp-sgd}
Abadi, M., Chu, A., Goodfellow, I., McMahan, H.B., Mironov, I., Talwar, K.,
  Zhang, L.: Deep learning with differential privacy. In: Proceedings of the
  2016 ACM SIGSAC conference on computer and communications security. pp.
  308--318 (2016)

\bibitem{double-geometric-2}
Balcer, V., Vadhan, S.: Differential privacy on finite computers. arXiv
  preprint arXiv:1709.05396  (2017)

\bibitem{wilcoxon-1}
Benavoli, A., Corani, G., Mangili, F.: Should we really use post-hoc tests
  based on mean-ranks? The Journal of Machine Learning Research
  \textbf{17}(1),  152--161 (2016)

\bibitem{dp-logreg3}
Bonte, C., Vercauteren, F.: Privacy-preserving logistic regression training.
  BMC medical genomics  \textbf{11},  13--21 (2018)

\bibitem{LIBSVM}
Chang, C.C., Lin, C.J.: Libsvm: a library for support vector machines. ACM
  transactions on intelligent systems and technology (TIST)  \textbf{2}(3),
  1--27 (2011)

\bibitem{dp-logreg1}
Chaudhuri, K., Monteleoni, C.: Privacy-preserving logistic regression. Advances
  in neural information processing systems  \textbf{21} (2008)

\bibitem{remark-1}
Clarkson, K.L.: Coresets, sparse greedy approximation, and the frank-wolfe
  algorithm. ACM Transactions on Algorithms (TALG)  \textbf{6}(4),  1--30
  (2010)

\bibitem{wilcoxon-2}
Dem{\v{s}}ar, J.: Statistical comparisons of classifiers over multiple data
  sets. The Journal of Machine learning research  \textbf{7},  1--30 (2006)

\bibitem{dp-definition}
Dwork, C., Kenthapadi, K., McSherry, F., Mironov, I., Naor, M.: Our data,
  ourselves: Privacy via distributed noise generation. In: Advances in
  Cryptology-EUROCRYPT 2006: 24th Annual International Conference on the Theory
  and Applications of Cryptographic Techniques, St. Petersburg, Russia, May
  28-June 1, 2006. Proceedings 25. pp. 486--503. Springer (2006)

\bibitem{post-processing-dp}
Dwork, C., Roth, A., et~al.: The algorithmic foundations of differential
  privacy. Foundations and Trends{\textregistered} in Theoretical Computer
  Science  \textbf{9}(3--4),  211--407 (2014)

\bibitem{frank-wolfe-original}
Frank, M., Wolfe, P.: An algorithm for quadratic programming. Naval research
  logistics quarterly  \textbf{3}(1-2),  95--110 (1956)

\bibitem{logreg-hessian}
Gasso, G.: Logistic regression. INSA Rouen-ASI Departement Laboratory:
  Saint-Etienne-du-Rouvray, France pp. 1--30 (2019)

\bibitem{double-geometric-1}
Ghosh, A., Roughgarden, T., Sundararajan, M.: Universally utility-maximizing
  privacy mechanisms. In: Proceedings of the forty-first annual ACM symposium
  on Theory of computing. pp. 351--360 (2009)

\bibitem{numpy}
Harris, C.R., Millman, K.J., Van Der~Walt, S.J., Gommers, R., Virtanen, P.,
  Cournapeau, D., Wieser, E., Taylor, J., Berg, S., Smith, N.J., et~al.: Array
  programming with numpy. Nature  \textbf{585}(7825),  357--362 (2020)

\bibitem{lasso-advantages}
Hastie, T., Tibshirani, R., Wainwright, M.: Statistical learning with sparsity:
  the lasso and generalizations. CRC press (2015)

\bibitem{remark-2}
Jaggi, M.: Revisiting frank-wolfe: Projection-free sparse convex optimization.
  In: International conference on machine learning. pp. 427--435. PMLR (2013)

\bibitem{logistic-regression}
Jurafsky, D., Martin, J.H.: Speech and Language Processing (2021)

\bibitem{dp-composition}
Kairouz, P., Oh, S., Viswanath, P.: The composition theorem for differential
  privacy. In: International conference on machine learning. pp. 1376--1385.
  PMLR (2015)

\bibitem{medical-machine-learning}
Khanna, A., Schaffer, V., G{\"u}rsoy, G., Gerstein, M.: Privacy-preserving
  model training for disease prediction using federated learning with
  differential privacy. In: 2022 44th Annual International Conference of the
  IEEE Engineering in Medicine \& Biology Society (EMBC). pp. 1358--1361. IEEE
  (2022)

\bibitem{two-stage}
Kifer, D., Smith, A., Thakurta, A.: Private convex empirical risk minimization
  and high-dimensional regression. In: Conference on Learning Theory. pp.
  25--1. JMLR Workshop and Conference Proceedings (2012)

\bibitem{financial-machine-learning}
Kim, H., Cho, H., Ryu, D.: Corporate default predictions using machine
  learning: Literature review. Sustainability  \textbf{12}(16), ~6325 (2020)

\bibitem{dp-logreg4}
Kim, M., Lee, J., Ohno-Machado, L., Jiang, X.: Secure and differentially
  private logistic regression for horizontally distributed data. IEEE
  Transactions on Information Forensics and Security  \textbf{15},  695--710
  (2019)

\bibitem{interpretable-ml}
Lundberg, S.M., Lee, S.I.: A unified approach to interpreting model
  predictions. Advances in neural information processing systems  \textbf{30}
  (2017)

\bibitem{probabilistic-machine-learning}
Murphy, K.P.: Probabilistic machine learning: an introduction. MIT press (2022)

\bibitem{constrained-lasso}
Osborne, M.R., Presnell, B., Turlach, B.A.: On the lasso and its dual. Journal
  of Computational and Graphical statistics  \textbf{9}(2),  319--337 (2000)

\bibitem{convex-lipschitz}
Shalev-Shwartz, S., et~al.: Online learning and online convex optimization.
  Foundations and Trends{\textregistered} in Machine Learning  \textbf{4}(2),
  107--194 (2012)

\bibitem{privacy-attacks}
Shokri, R., Stronati, M., Song, C., Shmatikov, V.: Membership inference attacks
  against machine learning models. In: 2017 IEEE symposium on security and
  privacy (SP). pp. 3--18. IEEE (2017)

\bibitem{private-frank-wolfe}
Talwar, K., Guha~Thakurta, A., Zhang, L.: Nearly optimal private lasso.
  Advances in Neural Information Processing Systems  \textbf{28} (2015)

\bibitem{scipy}
Virtanen, P., Gommers, R., Oliphant, T.E., Haberland, M., Reddy, T.,
  Cournapeau, D., Burovski, E., Peterson, P., Weckesser, W., Bright, J.,
  et~al.: Scipy 1.0: fundamental algorithms for scientific computing in python.
  Nature methods  \textbf{17}(3),  261--272 (2020)

\bibitem{dp-ight}
Wang, L., Gu, Q.: Differentially private iterative gradient hard thresholding
  for sparse learning. In: 28th International Joint Conference on Artificial
  Intelligence (2019)

\bibitem{admm}
Wang, P., Zhang, H.: Differential privacy for sparse classification learning.
  Neurocomputing  \textbf{375},  91--101 (2020)

\bibitem{dp-logreg2}
Yu, F., Rybar, M., Uhler, C., Fienberg, S.E.: Differentially-private logistic
  regression for detecting multiple-snp association in gwas databases. In:
  Privacy in Statistical Databases: UNESCO Chair in Data Privacy, International
  Conference, PSD 2014, Ibiza, Spain, September 17-19, 2014. Proceedings. pp.
  170--184. Springer (2014)

\end{thebibliography}

\newpage
\appendix
\section{Sparse Private Multinomial Logistic Regression}
\label{sec:multinomial}
\subsection{Introduction}
Multinomial logistic regression is an extension of the traditional logistic regression technique to produce multiclass predictions. Specifically, given features $\{\mathbf{x}_1, \ldots, \mathbf{x}_n \} \in \mathbb{R}^{p}$ and a one-hot encoded label $\{\mathbf{y}_1, \ldots \mathbf{y}_n \} \in \{0, 1\}^{K}$, LASSO-regularized multinomial logistic regression solves 
$$\widehat{\mathbf{W}} \in \argmin_{\mathbf{W} \in \mathbb{R}^{K \times p}: \ \lVert \mathbf{W} \rVert_1 \leq \lambda} \frac{1}{n} \sum_{i = 1}^{n} \sum_{k = 1}^{K} - y_{ik} \log \widehat{y}_{ik}$$
where $\widehat{y}_{ik} = \texttt{softmax}(\mathbf{W}\mathbf{x}_i)_k$ \cite{probabilistic-machine-learning}. 

\subsection{Differentially-Private Training Algorithm}

In this section, we present a variant of \hyperref[alg:private-frank-wolfe]{\texttt{Private LASSO}} which can be used for training multinomial logistic regression. In future sections, we analyze the privacy and utility of this algorithm. 

\begin{algorithm}[h]
    \caption{Multinomial Private LASSO}
    \label{alg:private-multinomial-frank-wolfe}
    \begin{algorithmic}
        \REQUIRE Privacy Parameters: $\epsilon > 0, 0 < \delta \leq 1$; Constraint Parameter: $\lambda > 0$; Iteration Parameter: $T$; Dataset: $D$ where $\lVert \mathbf{x}_i \rVert_\infty \leq 1$; Loss Function: $\mathcal{L}(\mathbf{W}; D) = \frac{1}{n} \sum_{i = 1}^{n} \mathcal{L}(\mathbf{W}; d_i)$ where $\mathcal{L}(\mathbf{W}; d_i)$ is $L$-Lipschitz with respect to the $L_1$ norm. 
        \STATE
        \STATE $S \gets \lambda * \{\text{Vertices of the Unit } L_1 \text{ Ball}\}$
        \STATE $\widehat{\mathbf{W}}^{(1)} \gets \mathbf{0}$
        \FOR{$t = 1$ to $T - 1$} 
            \STATE $\mathbf{G} \gets \nabla \mathcal{L}(\widehat{\mathbf{W}}^{(t)}; D)$
            \FOR{$i = 1$ to $p$}
                \FOR{$\mathbf{s} \in S$}
                    \STATE $\alpha_\mathbf{s} \gets \langle \mathbf{s}, \mathbf{g}_{:i} \rangle + \text{Lap} \left( \frac{\lambda L \sqrt{8T \log(K/\delta)}}{n \epsilon / K} \right)$
                \ENDFOR
                \STATE $\widetilde{\mathbf{w}}_{:i}^{(t)} \gets \argmin_{\mathbf{s} \in S} \alpha_s$
            \ENDFOR
            \STATE $\widehat{\mathbf{W}}^{(t + 1)} \gets (1 - \mu_t)\widehat{\mathbf{W}}^{(t)} + \mu_t \widetilde{\mathbf{W}}^{(t)} \text{ where } \mu_t = \frac{2}{t + 2}$
        \ENDFOR
        \STATE Output $\widehat{\mathbf{W}}^{(T)}$
    \end{algorithmic}
\end{algorithm}

\begin{theorem}
\label{thm:multinomial-lipschitzness}
When each data point $\mathbf{x}_i$ satisfies $\lVert \mathbf{x}_i \rVert_\infty \leq 1$, $\mathcal{L}(\mathbf{W}; d_i) = \sum_{k = 1}^{K} -y_{ik} \log \widehat{y}_{ik}$ has Lipschitz constant $2$ with respect to the $L_1$ norm. 
\end{theorem}
\begin{proof}
We will use the same approach as the proof of \hyperref[thm:lipschitzness]{Theorem 2.1}. First we show that the feasible set $\lVert \mathbf{W} \rVert_1 \leq \lambda$ is convex. 

Choose non-negative numbers $\alpha$ and $\beta$ such that $\alpha + \beta = 1$. For any $\mathbf{W}_1$ and $\mathbf{W}_2$ in the feasible set, $\lVert \alpha \mathbf{W}_1 + \beta \mathbf{W}_2 \rVert_1 \leq \lVert \alpha \mathbf{W}_1 \rVert_1 + \lVert \beta \mathbf{W}_2 \rVert_1 \leq \alpha \lambda + \beta \lambda = \lambda$. Thus the feasible set is convex, and we satisfy the conditions required to apply Shalev-Shwartz's theorem described in \hyperref[thm:lipschitzness]{Theorem 2.1}. 

It can be shown that $\mathcal{L}(\mathbf{W}; d_i) = \sum_{k = 1}^{K} -y_{ik} \log \widehat{y}_{ik}$ is convex and $\nabla \mathcal{L}(\mathbf{W}; d_i) = \mathbf{x}_i (\widehat{\mathbf{y}}_i - \mathbf{y}_i)^{\text{T}}$ \cite{probabilistic-machine-learning}. We want to find the maximum of $\lVert \mathbf{x}_i (\widehat{\mathbf{y}}_i - \mathbf{y}_i)^{\text{T}} \rVert_\infty$. By the definition of the matrix infinity norm, $\lVert \mathbf{x}_i (\widehat{\mathbf{y}}_i - \mathbf{y}_i)^{\text{T}} \rVert_\infty = \lVert \mathbf{x}_i \rVert_\infty \lVert (\widehat{\mathbf{y}}_i - \mathbf{y}_i) \rVert_1$. 

Note that $\lVert \widehat{\mathbf{y}}_i - \mathbf{y}_i \rVert_1 = \lvert \widehat{y}_{i1} - y_{i1} \rvert + \dots + \lvert \widehat{y}_{iK} - y_{iK} \rvert$. Without loss of generality, say $y_{iK} = 1$, meaning that $\mathbf{x}_i$ is of class $K$. Then the above expression simplifies to $\sum_{k = 1}^{K - 1} \widehat{y}_{ik} + 1 - \widehat{y}_{iK}$. Since all $\widehat{y}_{ik}$ must be non-negative and $\sum_{k = 1}^{K} \widehat{y}_{ik} = 1$, we know that $\sum_{k = 1}^{K - 1} \widehat{y}_{ik} + 1 - \widehat{y}_{iK} \leq 2$. This means that $\lVert \mathbf{x}_i (\widehat{\mathbf{y}}_i - \mathbf{y}_i)^{\text{T}} \rVert_\infty \leq 2$, and thus $\mathcal{L}(\mathbf{W}; d_i)$ has Lipschitz constant 2. 
\end{proof}

\subsection{Privacy}

Here we show that \hyperref[alg:private-multinomial-frank-wolfe]{Algorithm 3: Multinomial Private LASSO} is $(\epsilon, \delta)$-differentially private. 

\begin{theorem}
\hyperref[alg:private-multinomial-frank-wolfe]{Algorithm 3: Multinomial Private LASSO} is differentially private with parameters $(\epsilon, \delta)$. 
\end{theorem}

\begin{proof}
According to Talwar et al., each vector $\widehat{\mathbf{w}}_{k:}^{(T)}$ is $\left(\frac{\epsilon}{K}, \frac{\delta}{K}\right)$-differentially private \cite{private-frank-wolfe}. Since $\widehat{\mathbf{W}}^{(T)}$ publishes $K$ of these vectors, the composition theorem of differential privacy indicates that $\widehat{\mathbf{W}}^{(T)}$ is $(\epsilon, \delta)$-differentially private \cite{dp-composition}. 
\end{proof}

\newpage
\section{Implementing Sparse Private LASSO Logistic Regression}
\label{sec:binary-implementation}
The following is a python code framework which can be used to implement \hyperref[alg:private-frank-wolfe]{\texttt{Private LASSO}} and \hyperref[alg:sparsifier]{\texttt{Sparsifier}}. The use of linear algebra and sparse matrices makes the inner loop of \hyperref[alg:private-frank-wolfe]{\texttt{Private LASSO}} much more efficient. Note that \texttt{y} is expected to be a column vector with shape \texttt{(num\_samples, 1)}. 
\begin{minted}[
frame=lines,
framesep=2mm,
baselinestretch=1.2,
fontsize=\footnotesize,
linenos
]{python}
import numpy as np
from scipy import sparse


def nonprivate_lasso(X, y, iterations, l1_bound):
    n, d = X.shape

    constraint = l1_bound * sparse.vstack(
        (sparse.identity(d, format="csc"), -sparse.identity(d, format="csc")),
        format="csc",
    )
    w = np.zeros((d, 1))

    for t in range(1, iterations + 1):
        sigmoid_X = 1 / (1 + np.exp(-(X @ w)))
        grad = (1 / n) * (sigmoid_X - y).T @ X

        directional_deriv = constraint @ grad.T
        w_tilde = constraint[np.argmin(directional_deriv), :].T

        w = w + (2 / (t + 2)) * (w_tilde - w)

    return w


def private_lasso(X, y, iterations, l1_bound, epsilon, delta):
    n, d = X.shape

    constraint = l1_bound * sparse.vstack(
        (sparse.identity(d, format="csc"), -sparse.identity(d, format="csc")),
        format="csc",
    )
    w = np.zeros((d, 1))

    for t in range(1, iterations + 1):
        sigmoid_X = 1 / (1 + np.exp(-(X @ w)))
        grad = (1 / n) * (sigmoid_X - y).T @ X

        directional_deriv = constraint @ grad.T
        directional_deriv += np.random.laplace(
            scale=(l1_bound * np.sqrt(8 * iterations * np.log(1 / delta)))
            / (n * epsilon), size=(2 * d, 1))
        w_tilde = constraint[np.argmin(directional_deriv), :].T

        w = w + (2 / (t + 2)) * (w_tilde - w)

    return w

    
def sparsifier(X, y, iterations, l1_bound, alpha, beta, 
               epsilon_1, epsilon_2, delta, rho):
    
    w_np = nonprivate_frank_wolfe(X, y, 50000, l1_bound)
    
    c = np.clip(np.count_nonzero(w_np), alpha, beta)
    p = 1 - np.exp(-epsilon_1 / (beta - alpha))
    c = round(c + np.random.geometric(p) - np.random.geometric(p))
    c = np.clip(np.count_nonzero(w_np), alpha, beta)
    c = c * rho
    if c < 0:
        c = 0
    if c > X.shape[1]:
        c = X.shape[1]
    
    w_p = private_frank_wolfe(X, y, iterations, l1_bound, epsilon_2, delta)
    w_p = np.squeeze(np.asarray(w_p))
    abs_w_p = np.abs(w_p)
    sparse_w_p = np.zeros(abs_w_p.shape)
    if c > 0:
        ind = np.argpartition(abs_w_p, -1 * c)[-1 * c:]
        sparse_w_p[ind] = w_p[ind]
    
    return sparse_w_p
\end{minted}
\end{document}